\documentclass[runningheads,a4paper]{llncs}

\usepackage{amsmath}
\usepackage{amssymb}
\usepackage{graphicx}

\usepackage{algorithm}
\usepackage{algorithmic}

\usepackage{hyperref}

\usepackage{amsfonts}
\usepackage{float}
\usepackage{enumitem}

\floatstyle{ruled}
\newfloat{Program}{ht}{lop}

\newtheorem{lmm}{Lemma}
\newtheorem{crl}{Corollary}

\newcommand {\I} {\ensuremath {\mathbf{1\hspace{-5.5pt}1}}}

\urldef{\mailsa}\path|{dtolpin,jwvdm,brooks,fwood}@robots.ox.ac.uk|
\newcommand{\keywords}[1]{\par\addvspace\baselineskip
\noindent\keywordname\enspace\ignorespaces#1}

\renewcommand{\vec}[1]{\ensuremath{\boldsymbol{#1}}}

\newcommand{\y}{\ensuremath{\vec{y}}}

\begin{document}

\mainmatter  

\title{Output-Sensitive Adaptive Metropolis-Hastings for Probabilistic Programs}


%
%
\author{David Tolpin \and Jan Willem van de Meent \and Brooks Paige \and Frank Wood}


\institute{University of Oxford \\ Department of Engineering Science \\
\mailsa}

%
%

\toctitle{Lecture Notes in Computer Science}
\tocauthor{Authors' Instructions}
\maketitle

\begin{abstract}

We introduce an adaptive output-sensitive
Metro\-polis-Hast\-ings algorithm for probabilistic models
expressed as programs, Adaptive Lightweight
Metro\-polis-Hast\-ings (AdLMH). The algorithm extends
Light\-weight Metro\-polis-Hast\-ings (LMH) by adjusting the
probabilities of proposing random variables for modification to
improve convergence of the program output. We show that AdLMH
converges to the correct equilibrium distribution and compare
convergence of AdLMH to that of LMH on several test problems to
highlight different aspects of the adaptation scheme. We observe
consistent improvement in convergence on the test problems.

\end{abstract}

\section{Introduction}
    
One strategy for improving convergence of Markov Chain Monte Carlo (MCMC) samplers is through
online adaptation of the proposal distribution~\cite{AT08,AFM+11,RR09}.
An adaptation scheme must ensure that the sample sequence converges to the correct equilibrium
distribution. In a componentwise updating Metropolis-Hastings
MCMC, i.e. Metro\-polis-within-Gibbs~\cite{GL06,LRR13,LRY+05},
the proposal distribution can be decomposed into two components:
\begin{enumerate}
\item A stochastic schedule (probability distribution) for selecting
  the next random variable for modification.
\item The kernels from which new values for each of the variables are proposed.
\end{enumerate}
In this paper we concentrate on the first component---adapting the
schedule for selecting a variable for modification.

Our primary interest in this work is to improve MCMC methods for probabilistic
programming~\cite{GMR+08,GHNR14,MSP14,WVM14}. Probabilistic programming languages
facilitate development of general probabilistic models using the expressive
power of general programming languages. The goal of inference in such programs
is to reason about the posterior distribution over random variates that are
sampled during execution, conditioned on observed values that constrain
a subset of program expressions.

Lightweight Metropolis-Hastings (LMH) samplers \cite{WSG11} propose a change to
a single random variable at each iteration. The program is then rerun, reusing
previous values and computation where possible, after which the new set of sample values
is accepted or rejected. While re-running the program each time may waste some
computation, the simplicity of LMH makes developing probabilistic
variants of arbitrary languages relatively straightforward.

Designing robust Adaptive MCMC methods for probabilistic programming is complicated
because of diversity of models expressed by probabilistic programs. The same adaption
scheme should perform well with different programs without manual tuning.
Here we present an adaptive variant of LMH, which dynamically adjusts the schedule for
selecting variables for modification. First, we review the general
structure of a probabilistic program. We discuss convergence criteria with
respect to the program output and propose a scheme for tracking the ``influence''
of each random variable on the output. We then adapt the selection probability
for each variable, borrowing techniques from the upper confidence bound (UCB) family of algorithms for
multi-armed bandits \cite{ANF02}. We show that the proposed adaptation scheme
preserves convergence to the target distribution under reasonable assumptions.
Finally, we compare original
and Adaptive LMH on several test problems to show how convergence is improved
by adaptation. 

\section{Preliminaries}

\subsection{Probabilistic Program}

A probabilistic program is a stateful deterministic computation
$\mathcal{P}$ with the following properties:
\begin{itemize}
\item Initially, $\mathcal{P}$ expects no arguments.
\item On every call, $\mathcal{P}$ returns either a distribution $F$, a
    distribution and a value $(G, y)$, a value $z$, or $\bot$.
\item Upon returning $F$, $\mathcal{P}$ expects a value $x$ drawn from $F$
as the argument to the next call.
\item Upon returning $(G, y)$ or $z$, $\mathcal{P}$ is invoked again
    without arguments.
\item Upon returning $\bot$, $\mathcal{P}$ terminates.
\end{itemize}
A program is run by calling $\mathcal{P}$ repeatedly until termination.

Every run of the program implicitly produces a sequence of
pairs $(F_i, x_i)$ of distributions and values
of \textit{latent} random variables. We call this
sequence a \textit{trace} and denote it by $\pmb{x}$. 
A trace induces a sequence of pairs
$(G_j, y_j)$ of distributions and values of
\textit{observed} random variables. We call this
sequence an \textit{image} and denote it by $\pmb{y}$. We call a sequence of
values $z_k$ an \textit{output} of the program and denote it by
$\pmb{z}$. Program output is deterministic given the trace.

The probability of a trace is proportional to the product of the
probability of all random choices $\pmb{x}$ and the likelihood of all
observations
$\pmb{y}$
\begin{equation}
	p(\pmb{x}) \propto \prod_{i=1}^{\left|\pmb{x}\right|}
    p_{F_i}(x_i) \prod_{j=1}^{\left|\pmb{y}\right|}p_{G_j}(y_{j}).
  \label{eqn:p-trace}
\end{equation}
The objective of inference in probabilistic program $\mathcal{P}$
is to discover the distribution of $\pmb{z}$.

\subsection{Adaptive Markov Chain Monte Carlo}

MCMC methods generate a sequence of samples
$\{\pmb{x}^{t}\}_{t=1}^{\infty}$ by simulating a Markov chain
using a transition operator that leaves a
target density $\pi(\pmb{x})$ invariant. In MH the transition operator is
implemented by drawing a new sample $\pmb{x}'$  from a parameterized
proposal distribution $q_\theta(\pmb{x}'|\pmb{x}^{t})$ that is
conditioned on the current sample $\pmb{x}^t$. The proposed sample is then
accepted with probability
\begin{equation}
    \rho=
    \min\left(\frac {\pi(\pmb{x}')q_\theta(\pmb{x}^{t}|\pmb{x}')}
                    {\pi(\pmb{x}^{t})q_\theta(\pmb{x}'|\pmb{x}^{t})},
              1\right)
    .
    \label{eqn:mh-rho}
\end{equation}
If $\pmb{x}'$ is rejected, $\pmb{x}^t$ is re-used as the
next sample.

The convergence rate of MH depends on parameters $\theta$ of the
proposal distribution $q_\theta$. The parameters can be set
either offline or online. Variants of MCMC in which the parameters
are continuously adjusted based on the features of the sample sequence
are called adaptive. Challenges in design and analysis of
Adaptive MCMC methods include optimization criteria and algorithms
for the parameter adaptation, as well as conditions of convergence of
Adaptive MCMC to the correct equilibrium distribution \cite{RR07}.
Continuous adaptation of parameters of the proposal distribution is 
a well-known research subject~\cite{AT08,AFM+11,RR09}.

In a componentwise MH algorithm~\cite{LRY+05} which targets a density 
$\pi(\pmb{x})$ defined on an $N$-dimensional space $\mathcal{X}$,
the components of a random 
sample $\pmb{x} = \{x_1, \ldots, x_N\}$ are updated individually, in either
random or systematic order. Assuming the component $i$ is selected at the
step $t$ for modification, the proposal $\pmb{x}'$ sampled from $q^i_\theta(\pmb{x}|\pmb{x}^t)$
may differ
from $\pmb{x}^t$ only in that component, and $x_j' =x_j^{t}$ for all $j\ne
i$. Adaptive componentwise Metropolis-Hastings (Algorithm~\ref{alg:amh})
chooses different probabilities for selecting a component for modification at
each iteration. Parameters of this scheduling distribution may be viewed as a
subset of parameters $\theta$ of the proposal distribution $q_\theta$, and
adjusted according to optimization criteria of the sampling algorithm.

\begin{algorithm}[t]
    \begin{algorithmic}[1]
        \STATE Select initial point $\pmb{x}^1$.
        \STATE Set initial selection probabilities $\pmb{\alpha}^1$.
        \FOR {$t = 1 \ldots \infty$}
        \STATE $\pmb{\alpha}^t \gets f^t(\pmb{\alpha}^{t},
        \pmb{x}^0, \pmb{x}^1, \ldots, \pmb{x}^{t})$.
          \STATE Choose $k \in \{1,\ldots,N\}$ with probability $\alpha_k^t$.
          \STATE Generate $\pmb{x}' \sim q^k_\theta(\pmb{x}|\pmb{x}^{t})$.
          \STATE $\rho \gets 
                  \min\left(\frac {\pi(\pmb{x}')q^k_\theta(\pmb{x}^{t}|\pmb{x}')}
                      {\pi(\pmb{x}^{t})q^k_\theta(\pmb{x}'|\pmb{x}^{t})},
                  1\right)$
          \STATE $\pmb{x}^{t+1} \gets$ $\pmb{x}'$ with probability $\rho$,
                                       $\pmb{x}^{t}$ otherwise.
        \ENDFOR
    \end{algorithmic}
    \caption{Adaptive componentwise MH}
    \label{alg:amh}
\end{algorithm}

Varying selection probabilities based on past samples violates the
Markov property of $\{{\pmb{x}^t}\}_1^\infty$. 
However, provided the adaptation of the selection probabilities diminishes, 
with $||\pmb{\alpha}^t -\pmb{\alpha}^{t-1}|| \rightarrow 0$,
then under suitable regularity conditions for the target density
(see Section~\ref{sec:convergence})
an adaptive componentwise MH algorithm will still be ergodic \cite{LRR13},
and the distribution on $\pmb{x}$ induced by Algorithm~\ref{alg:amh} 
converges to $\pi$.

\subsection{Lightweight Metropolis-Hastings}

LMH~\cite{WSG11} is a sampling scheme for probabilistic programs where a single
random variable drawn in the course of a particular execution of a
probabilistic program is modified via a standard MH proposal, and this
modification is accepted by comparing the values of the joint
probability of old and new program traces. LMH differs from
componentwise MH algorithms in that other random variables may also
have to be modified, depending on the structural dependencies in
the probabilistic program.

LMH initializes a proposal by selecting a single variable $x_k$ from an
execution trace ${\pmb x}$ and resampling its value $x'_k$ either using a
reversible kernel $\kappa(x'_k|x_k)$ or from the conditional prior distribution.
Starting from this
initialization, the program is rerun to generate a new trace ${\pmb x}'$. For
each $m > k$, the previous value $x_m$ is reused, provided it still lies
in the support of the distribution on $x'_m$, rescoring its $\log$
probability relative to the new random choices $\{x'_1, \ldots,
x'_{m-1}\}$. When $x_m'$ cannot be rescored, a new value is sampled
from the prior on $x'_m$, conditioned on preceding choices. The acceptance
probability $\rho_\mathrm{LMH}$ is obtained by substituting 
(\ref{eqn:p-trace}) into (\ref{eqn:mh-rho}):
\begin{equation}
    \rho_\mathrm{LMH}=\min \left(1,
        \frac {p(\pmb{y}'|\pmb{x}') p(\pmb{x}')q(\pmb{x}|\pmb{x}')}
            {p(\pmb{y }|\pmb{x }) p(\pmb{x })q(\pmb{x}'|\pmb{x})}
                           \right)
    .
\label{eqn:rho-lmh}
\end{equation}
We here further simplify LMH by assuming $x'_k$ is sampled from the prior conditioned on
earlier choices and that all variables are selected for modification with equal
probability. In this case, $\rho_\mathrm{LMH}$ takes the form~\cite{WVM14}
\begin{equation}
    \rho_\mathrm{LMH}=\min \left(1,
        \frac {p(\pmb{y}'|\pmb{x}') p(\pmb{x}')|\pmb{x}|
               p(\pmb{x}\!\setminus\!\pmb{x}'|\pmb{x} \cap \pmb{x}')}
              {p(\pmb{y}|\pmb{x}) p(\pmb{x})|\pmb{x}'|
              p(\pmb{x}'\!\setminus\!\pmb{x}|\pmb{x}' \cap \pmb{x})}
           \right)
           ,
\label{eqn:rho-lmh-simple}
\end{equation}
where $\pmb{x}'\!\setminus\!\pmb{x}$ denotes the resampled variables,
and $\pmb{x}' \cap \pmb{x}$ denotes the variables which have the same
values in both traces.

\section{Adaptive Lightweight Metropolis-Hastings}

We develop an adaptive variant of LMH, which dynamically adjusts the probabilities
of selecting variables for modification (Algorithm~\ref{alg:adlmh}).
Let $\pmb{x}^t$ be the trace at iteration $t$ of Adaptive LMH.
We define the probability distribution of selecting variables for modification
in terms of a weight vector $\pmb{W}^t$ that we adapt, such
that the probability $\alpha^t_i$ of selecting the $i^{th}$ variable for
modification is
\begin{equation}
    \alpha^t_i=\frac {W_i^t} {\sum\limits_{k=1} ^{|\pmb{x}^t|} W_k^t}.
\label{eqn:alpha-def}
\end{equation}

\begin{algorithm}[t]
\begin{algorithmic}[1]
    \STATE Initialize $\pmb{W}^0$ to a constant.
\STATE {Run the program.}
\FOR {$t = 1 \ldots \infty$}
  \STATE {Randomly select a variable $x^t_k$ according to $\pmb{W}^t$.} \label{alg:adlmh-propose}
  \STATE {Propose a value for $x^t_k$.}
  \STATE {Run the program, accept or reject the trace.}
  \IF {accepted}
      \STATE {Compute $\pmb{W}^{t+1}$ based on the program output.} \label{alg:adlmh-compute-w}
  \ELSE
      \STATE $\pmb{W}^{t+1} \gets \pmb{W}^{t}$
  \ENDIF
\ENDFOR
\end{algorithmic}
\caption{Adaptive LMH}
\label{alg:adlmh}
\end{algorithm}

Just like LMH, Adaptive LMH runs the probabilistic program once and then
selects variables for modification randomly.
However, acceptance ratio $\rho_\mathrm{AdLMH}$ must now include
selection probabilities $\alpha_k$ and $\alpha_k'$ of the resampled
variable in the current and the proposed sample

\begin{equation}
    \rho_\mathrm{AdLMH}=\min \left(1,
        \frac {p(\pmb{y}'|\pmb{x}') p(\pmb{x}') \alpha_k'
               p(\pmb{x}\!\setminus\!\pmb{x}'|\pmb{x} \cap \pmb{x}')}
              {p(\pmb{y}|\pmb{x}) p(\pmb{x}) \alpha_k
              p(\pmb{x}'\!\setminus\!\pmb{x}|\pmb{x}' \cap \pmb{x})}
           \right)
    .
\label{eqn:rho-almh}
\end{equation}

This high-level description of the algorithm does not detail how $\pmb{W}^t$ is
computed for each iteration. Indeed, this is the most essential part of
the algorithm. There are two different aspects here --- on one hand, the
influence of a given choice on the output sequence must be quantified in
terms of convergence of the sequence to the target distribution. On the other
hand, the influence of the choice must be translated into re-computation of
weights of random variables in the trace. Both parts of re-computation of $\pmb{W}^t$
are explained below.

\subsection{Quantifying Influence}

Extensive research literature is available on criteria for tuning parameters of
Adaptive MCMC~\cite{AT08,AFM+11,RR09}. The case of inference in probabilistic
programs is different: the user of a probabilistic program is interested
in fast convergence of the program output $\{\pmb{z}^t\}$ rather than of the
vector of the program's random variables $\{\pmb{x}^t\}$.

In adaptive MCMC variants the acceptance rate can be efficiently used as the
optimization objective~\cite{RR09}. However, for convergence of the output
sequence an accepted trace that produces the same output is indistinguishable
from a rejected trace. Additionally, while optimal values of the acceptance
rate~\cite{AT08,RR09} can be used to tune parameters in adaptive MCMC, in
Adaptive LMH we do not change the parameters of proposal distributions of
individual variables, and assume that they are fixed. However, proposing
a new value for  a random variable may or may not change the output even
if the new trace is accepted. By changing variable selection
probabilities we attempt to maximize the change in the output sequence
so that it converges faster. In the pedagogical example
\begin{align*}
	x_1\sim& \mathrm{Bernoulli}(0.5),\quad x_2\sim \mathcal{N}(x_1, 1),\\
	z_1 \gets& (x_1, x_2),
\end{align*}
selecting the Bernoulli random choice for modification changes the output only
when a different value is sampled, while selecting the normal random choice
will change the output almost always.

Based on these considerations, we quantify the influence of sampling on
the output sequence by measuring the change in the output $\pmb{z}$ of
the probabilistic program. Since probabilistic programs may produce
output of any type, we chose to discern between identical and different
outputs only, rather than to quantify the distance by introducing a
type-dependent norm. In addition, when $|\pmb{z}|>1$, we quantify the
difference by the fraction of components of $\pmb{z}$  with changed values.

Formally, let
$\{\pmb{z}^t\}_1^\infty=\{\pmb{z}^1,\ldots,\pmb{z}^{t-1},\pmb{z}^t,\ldots\}$
be the output sequence of a probabilistic program. Then the
\textit{influence} of a choice that produced $\pmb{z}^t$ is defined by the
total reward $R^t$, computed as normalized Hamming distance between
the outputs
\begin{equation}
    R^t = \frac 1 {|\pmb{z}^t|}\sum_{i=k}^{|\pmb{z}^t|} \I(z_k^t \ne z_k^{t-1}).
\label{eqn:sample-reward}
\end{equation}
In the case of a scalar $\pmb{z}$, the reward is 1 when outputs of subsequent
samples are different, 0 otherwise.

The reward is used to adjust the variable selection probabilities for the
subsequent steps of Adaptive LMH by computing $\pmb{W}^{t+1}$
(line~\ref{alg:adlmh-compute-w} of Algorithm~\ref{alg:adlmh}). It may seem
sufficient to assign the reward to the last choice and use average choice
rewards as their weights, but  this approach will not work for
Adaptive LMH.  Consider the generative model
\begin{align*}
	x_1 \sim& \mathcal{N}(1, 10), \quad x_2 \sim \mathcal{N}(x_1, 1), \\
	y_1 \sim& \mathcal{N}(x_2, 1), \\
	z_1 \gets& x_1,
\end{align*}
where we observe the value $y_1 = 2$. Modifying $x_2$ may result in an accepted trace, but the value of
$z_1=x_1$, predicted by the program, will remain the same as in the
previous trace. Only when $x_1$ is also modified, and a new trace with
the updated values for both $x_1$ and $x_2$ is accepted, the
earlier change in $x_2$ is indirectly reflected in the output of the program.
In the next section, we discuss propagation of rewards to variable
selection probabilities in detail.

\subsection{Propagating Rewards to Variables}

Both LMH and Adaptive LMH modify a single variable per trace, and either
re-use or recompute the probabilities of values of all other variables
(except those absent from the previous trace or having an incompatible
distribution, for which new values are also sampled).  Due to this
updating scheme, the influence of modifying a variable on the
output can be delayed by several iterations. We propose the following
propagation scheme: for each random variable $x_i$, the 
reward $r_i$ and count $c_i$ are kept in a
data structure used to compute $\pmb{W}$.  The set of variables
selected for modification, called here the \textit{history}, is
maintained for each component $z_k$ of output $\pmb{z}$.  When the
value of $z_k$ changes, the reward is distributed between all of the
variables in the history, and the history is emptied. When $z_k$
does not change, the selected variable is penalized by zero reward.
This scheme, for the case of scalar output for simplicity, is shown
in Algorithm~\ref{alg:award} which expands
line~\ref{alg:adlmh-compute-w} of Algorithm~\ref{alg:adlmh}.  When
$\pmb{z}$ has multiple components, histories for each component are
maintained independently.

\begin{algorithm}
    \begin{algorithmic}[1]
        \STATE Append $x_k$ to the history of variables selected
		for modification.
        \IF {$\pmb{z}^{t+1} \ne \pmb{z}^{t}$}
			\STATE $w \gets \frac 1 {|history|}$
            \FOR {$x_m$ {\bf in} history}
				\STATE $\overline r_m \gets r_m + w,\; c_m \gets c_m + w$ \label{alg:award-reward}
            \ENDFOR
			\STATE Flush the history. \label{alg:award-flush}
        \ELSE
            \STATE $c_k \gets c_k + 1$ \label{alg:award-penalty}
        \ENDIF
    \end{algorithmic}
    \caption{Propagating Rewards to Variables}
    \label{alg:award}
\end{algorithm}

Rewarding all of the variables in the history ensures that while
variables which cause changes in the output more often get a greater
reward, variables with lower influence on the output are still selected
for modification sufficiently often. This, in turn, ensures ergodicity
of sampling sequence, and helps establish conditions for convergence to
the target distribution, as we discuss in Section~\ref{sec:convergence}.

Let us show that under certain assumptions the proposed reward
propagation scheme has a non-degenerate equilibrium for variable
selection probabilities.  Indeed, assume that for a program with
two variables, $x_1$, and $x_2$, {\it probability matching}, or
selecting a choice with the probability proportional to the
unit reward $\rho_i = \frac {r_i} {c_i}$, is used to compute
the weights, that is, $W_i=\rho_i$. Then, the following
lemma holds:
\begin{lmm}
	Assume that for variables $x_i$, where $i \in \{1, 2\}$:
\begin{itemize}
 \item $\alpha_i$ is the selection probability;
 \item $\beta_i$ is the probability that the new trace is accepted
     given that the variable was selected for modification;
 \item  $\gamma_i$ is the probability that the output changed
     given that the trace was accepted.
\end{itemize}
Assume further that $\alpha_i$, $\beta_i$, and $\gamma_i$ are constant.
Then for the case $\gamma_1=1$, $\gamma_2=0$:
\begin{equation}
	0 < \frac {\alpha_2} {\alpha_1} \le \frac 1 3 
    \label{eqn:lemma-rate}
\end{equation}
\label{thm:lemma-rate}
\end{lmm}
\begin{proof}
	We shall proof the lemma in three steps. First, we shall
	analyse a sequence of samples between two subsequent
	arrivals of $x_1$. Then, we shall derive a
	formula for the expected unit reward of $x_2$. Finally, we shall
	bound the ratio $\frac {\alpha_2} {\alpha_1}$.  
	
	Consider a sequence of $k$ samples, for some $k$, between two subsequent
	arrivals of $x_1$, including the sample corresponding to the
	second arrival of $x_1$. Since a new value of $x_1$
	always ($\gamma_1=1$) and $x_2$ never ($\gamma_2=0$) causes
	a change in the output, at the end of the sequence the
	history will contain $k$ occurrences of $x_2$. Let us denote 
	by $\Delta r_i$, $\Delta c_i$ the increase of reward $r_i$ and
	count $c_i$ between the beginning and the end of the
	sequence. Noting that $x_2$ is penalized each time it is
	added to the history (line~\ref{alg:award-penalty} of
	Algorithm~\ref{alg:award}), and $k$ occurrences of $x_2$ are
	rewarded when $x_1$ is added to the history
	(line~\ref{alg:award-reward} of Algorithm~\ref{alg:award}),
	we obtain
	\begin{equation}
		\Delta r_1=\frac 1 {k+1},\,\Delta c_1=\frac 1 {k+1}\quad\Delta r_2=\frac k {k+1},\,\Delta c_2=k+\frac k {k+1}
		\label{eqn:lemma-rate-rewards}
	\end{equation}
	Consider now a sequence of $M$ such sequences. When
	$M\to\infty$, $\frac {r_{iM}} {c_{iM}}$ approaches the expected
	unit reward $\overline \rho_i$, where $r_{iM}$ and $c_{iM}$ are the reward
	and the count of $x_i$ at the end of the sequence.
	\begin{equation}
		\overline \rho_i = \lim_{M\to \infty} \frac {r_{iM}} {c_{iM}} = \lim_{M\to \infty} \frac {\frac {r_{iM}} M} {\frac {c_{iM}} M} = \lim_{M\to \infty} \frac {\frac {\sum_{m=1}^M \Delta {r_{im}}} M} {\frac {\sum_{m=1}^M \Delta {c_{im}}} M} = \frac {\overline {\Delta r_i}} {\overline {\Delta c_i}}
		\label{eqn:lemma-rate-mean-rho}
	\end{equation}

	Each variable $x_i$ is selected randomly and independently and
	produces an accepted trace with probability
	\begin{equation}
		p_i= \frac {\alpha_i\beta_i} {\alpha_1\beta_1 + \alpha_2\beta_2}.
		\label{eqn:lemma-rate-p}
	\end{equation}
	Acceptances of $x_1$ form a Poisson process with rate $\frac 1
	{p_1}= \frac {\alpha_1\beta_1+\alpha_2\beta_2}
	{\alpha_1\beta_1}$. $k$ is distributed according to the
	geometric distribution with probability $p_1$, $\Pr[k] = (1-p_1)^kp_1$.
	Since $\Delta r_1 = \Delta c_1$ for any $k$, the expected unit reward $\overline
	\rho_1$ of $x_1$ is $1$. We shall substitute $\overline {\Delta
	r_i}$ and $\overline {\Delta c_i}$ into
	(\ref{eqn:lemma-rate-mean-rho}) to obtain the expected unit
	reward $\overline \rho_2$ of $x_2$:
	\begin{align}
		\overline {\Delta r_2} =& \sum_{k=0}^\infty \frac k {k+1} (1-p_1)^kp_1\nonumber\\
		\overline {\Delta c_2} =& \sum_{k=0}^\infty \left(k+\frac k {k+1}\right) (1-p_1)^kp_1=\underbrace{\frac {1-p_1} {p_1}}_{\overline k} + \sum_{k=0}^\infty  \frac k {k+1} (1-p_1)^kp_1
		\label{eqn:lemma-rate-delta}
	\end{align}
	\begin{align}
		\overline \rho_2 = &\frac {\overline {\Delta r_2}} {\overline {\Delta c_2}} = \frac {\sum\limits_{k=0}^\infty \frac k {k+1} (1-p_1)^kp_1} {\frac {1-p_1} {p_1} +  \sum\limits_{k=0}^\infty  \frac k {k+1} (1-p_1)^kp_1} = \frac {1-\overbrace{\sum_{k=0}^\infty \frac 1 {k+1} (1-p_1)^kp_1}^A} {\frac {1} {p_1} -  \underbrace{\sum_{k=0}^\infty  \frac 1 {k+1} (1-p_1)^kp_1}_A}
		\label{eqn:lemma-rate-expected-reward}
	\end{align}
	For probability matching, selection probabilities are
	proportional to expected unit rewards:
	\begin{equation}
		\frac {\alpha_2} {\alpha_1} = \frac {\overline \rho_2} {\overline \rho_1}
		\label{eqn:lemma-rate-matching}
	\end{equation}
	To prove the inequality, we shall derive a closed-form
	representation for $\overline \rho_2$, and analyse solutions of
	(\ref{eqn:lemma-rate-matching}) for $\frac {\alpha_2}
	{\alpha_1}$. We shall eliminate the summation $A$ in ($\ref{eqn:lemma-rate-expected-reward}$):
	\begin{align}
		A=&\sum_{k=0}^\infty  \frac k {k+1} (1-p_1)^kp_1=\frac {p_1} {1-p_1} \sum_{k=0}^\infty \frac 1 {k+1} (1-p_1)^{k+1}\nonumber\\
		=&\frac {p_1} {1-p_1} \sum_{k=0}^\infty \int_{p_1}^1 (1-\xi)^k d\xi = \frac{p_1} {1-p_1} \int_{p_1}^1 \sum_{k=0}^\infty(1-\xi)^kd\xi=-\frac {p_1} {1-p_1} \log p_1
		\label{eqn:lemma-rate-A}
	\end{align}
	By substituting $A$ into (\ref{eqn:lemma-rate-expected-reward}), and then
	$\overline \rho_1$ and $\overline \rho_2$ into
	(\ref{eqn:lemma-rate-matching}), we obtain 
	\begin{equation}
	\frac {\alpha_2} {\alpha_1} = \frac {\overline \rho_2} {\overline \rho_1} = \overline \rho_2 =  \left. \frac {1+\frac {p_1 \log p_1} {1-p_1}} {\frac 1 {p_1}+\frac {p_1 \log p_1} {1-p_1}} \right\rbrace B(p_1)
		\label{eqn:lemma-rate-equilibrium}
	\end{equation}
	The right-hand side $B(p_1)$ of
	(\ref{eqn:lemma-rate-equilibrium}) is a monotonic function
	for $p_1 \in [0, 1]$, and $B(0)=0$, $B(1)=\frac 1 3$ (see
	Appendix for the analysis of $B(p_1)$). According to
	(\ref{eqn:lemma-rate-p}), $\frac {\alpha_2} {\alpha_1}=0$
	implies $p_1=1$, hence $\frac {\alpha_2} {\alpha_1}\ne 0$,
	and $0 < \frac {\alpha_2} {\alpha_1} \le \frac 1 3$. \qed
\end{proof}

By noting that any subset of variables in a probabilistic program can be
considered a single random variable drawn from a multi-dimensional
distribution, Lemma~\ref{thm:lemma-rate} is generalized to any
number of variables by Corollary~\ref{thm:corollary-partitioning}:
\begin{crl}
    For any partitioning of the set $\pmb{x}$  of random variables of a
	probabilistic program, AdLMH with weights proportional to expected unit
	rewards selects variables from each of the partitions with non-zero
	probability.
    \label{thm:corollary-partitioning}
\end{crl}

To ensure convergence of $\pmb{W}^{t}$ to expected unit rewards in the stationary
distribution, we use upper confidence bounds on unit rewards to compute the variable
selection probabilities, an idea which we borrowed from the UCB family
of algorithms for multi-armed bandits~\cite{ANF02}. 
Following UCB1~\cite{ANF02}, we
compute the upper confidence bound $\hat \rho_i$ as the sum of the unit
reward and the exploration term
\begin{equation}
    \hat \rho_i = \rho_i + C\sqrt {\frac {\log \sum_{i=1}^{|\pmb{x}|}c_i} {c_i}}
    ,
\label{eqn:ucb-reward}
\end{equation}
where $C$ is an exploration factor. The default value for $C$ is $\sqrt 2$
in UCB1; in practice, a lower value of $C$ is preferable.
Note that variable selection in Adaptive LMH is different
from arm selection in multi-armed bandits: unlike in bandits,
where we want to sample the best arm at an increasing rate, in
Adaptive LMH we expect $\pmb{W}^t$ to converge to an equilibrium in which
selection probabilities are proportional to expected unit rewards.

\section{Convergence of Adaptive LMH}
\label{sec:convergence}

As adaptive MCMC algorithms may depend arbitrarily on the history
at each step, showing that a given sampler correctly draws from the
target distribution can be non-trivial. 
General conditions under which adaptive MCMC schemes are still
ergodic, in the sense that the distribution of samples converges to
the target $\pi$ in total variation, are established in \cite{RR07}.
The fundamental criteria for validity of an adaptive algorithm are
{\em diminishing adaptation}, which (informally) requires that the
amount which the transition operator changes each iteration must
asymptotically decrease to zero;
and {\em containment}, a technical condition which requires that
the time until convergence to the target distribution must be bounded
in probability \cite{BRR11}.

The class of models representable by probabilistic programs is very
broad, allowing specification of completely arbitrary target densities;
however, for many models the adaptive LMH algorithm reduces to an
adaptive random scan Metro\-polis-within-Gibbs in Algorithm~\ref{alg:amh}. 
To discuss when this is the case, we invoke 
the concept of {\em structural} versus {\em structure-preserving} 
random choices \cite{YHG14}. Crucially, a {\em structure-preserving} 
random choice $x_k$ does not affect the existence of other $x_m$
 in the trace. 

Suppose we were to restrict the expressiveness of our language to admit only 
programs with no structural random choices: in such a language, the LMH 
algorithm in Algorithm~\ref{alg:adlmh} reduces to the adaptive componentwise 
MH algorithm. Conditions under which such an adaptive 
algorithm is ergodic have been established explicitly in \cite[Theorems 4.10 and 5.5]{LRR13}. 
Given suitable assumptions on the target density defined by the program, 
it is necessary for the probability vector $||\alpha^t - \alpha^{t-1}|| \rightarrow 0$, 
and that for any particular component $k$ we have probability $\alpha^t_k > \epsilon > 0$. 
Both of these are satisfied by our approach: from Corollary~\ref{thm:corollary-partitioning},
we ensure that the unit reward across each $x_i$ converges to a positive fixed point.

While any theoretical result will  require language restrictions
such that programs only induce distributions satisfying regularity
conditions, we conjecture that this scheme is broadly applicable across
most non-pathological programs.
We leave a precise theoretical analysis of
the space of probabilistic programs in which adaptive MCMC 
schemes (with infinite adaptation) may be ergodic to future work.
Empirical evaluation presented in the next section demonstrates
practical convergence of Adaptive LMH on a range of inference
examples, including programs containing structural random choices.

\section{Empirical Evaluation}

We evaluated Adaptive LMH on many probabilistic programs and observed consistent
improvement of convergence rate compared to LMH. We also verified on a number
of tests that the algorithm converges to the correct distribution obtained by 
independent exact methods. In this section, we compare Adaptive LMH to LMH
on several representative examples of
probabilistic programs.  The rates in the comparisons are presented
with respect to the number of samples, or simulations, of the
probabilistic programs. The additional computation required for
adaptation takes negligible time, and the computational effort per
sample is approximately the same for all algorithms. Our
implementation of the inference engine is available at
\url{https://bitbucket.org/dtolpin/embang}.

In the following case studies differences
between program outputs and target distributions are presented using
KullBack-Leibler (KL) divergence, Kolmogorov-Smirnov (KS) distance,
or L2 distance, as appropriate. In cases where target distributions
cannot be updated exactly, they were approximated by running a
non-adaptive inference algorithm for a long enough time and with a
sufficient number of restarts.  In each of the evaluations, all of
the algorithms were run with 25 random restarts and 500\,000
simulations of the probabilistic program per restart. The difference
plots use the logarithmic scale for both axes. In the plots, the
solid lines correspond to the median, and the dashed lines to 25\%
and 75\% percentiles, taken over all runs of the corresponding
inference algorithm. The exploration factor for computing upper
confidence bounds on unit rewards (Equation~\ref{eqn:ucb-reward})
was fixed at $C=0.5$ for all tests and evaluations.

The first example is a latent state
inference problem in an HMM with three states, one-dimensional
normal observations (0.9, 0.8, 0.7, 0, -0.025, 5, 2, 0.1, 0, 0.13,
0.45, 6, 0.2, 0.3, -1, -1) with variance 1.0, a known transition
matrix, and known initial state distribution. There are 18 distinct
random choices in all traces of the program, and the
0th and the 17th state of the model are predicted. The results of
evaluation are shown in Figure~\ref{fig:hmm-x} as KL divergences
between the inference output and the ground truth obtained using the
forward-backward algorithm. In addition, bar plots of unit reward and
sample count distributions among random choices in Adaptive LMH are
shown for  $1000$, $10\,000$, and $100\,000$ samples.

\begin{figure}[t]
	\begin{minipage}[c]{0.45\linewidth}
		\centering
		\includegraphics[scale=0.4]{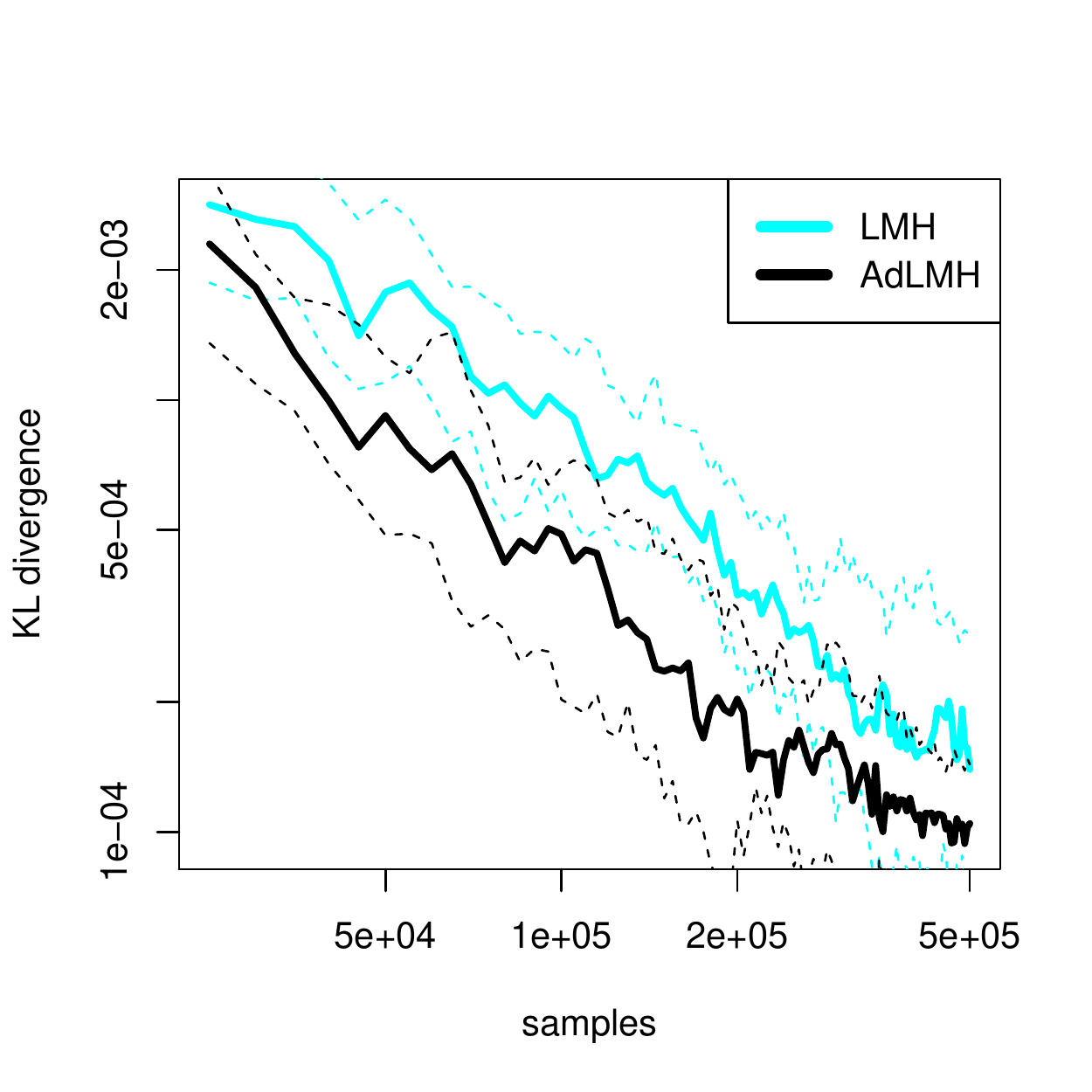} \\
	\end{minipage}
	\begin{minipage}[c]{0.55\linewidth}
		\centering
		\includegraphics[scale=0.3]{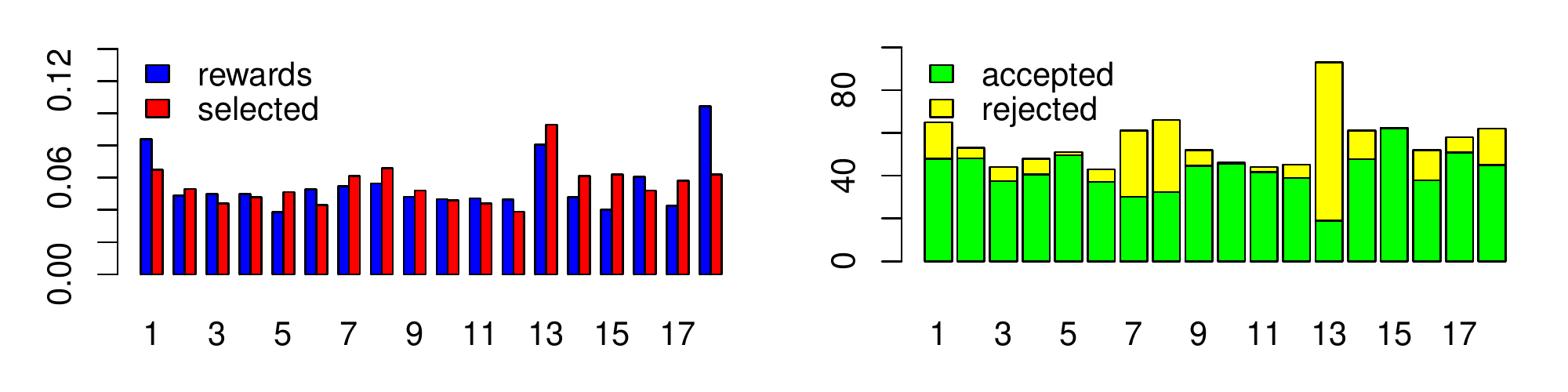} \\
		{\scriptsize 1000 samples} \\
		\includegraphics[scale=0.3]{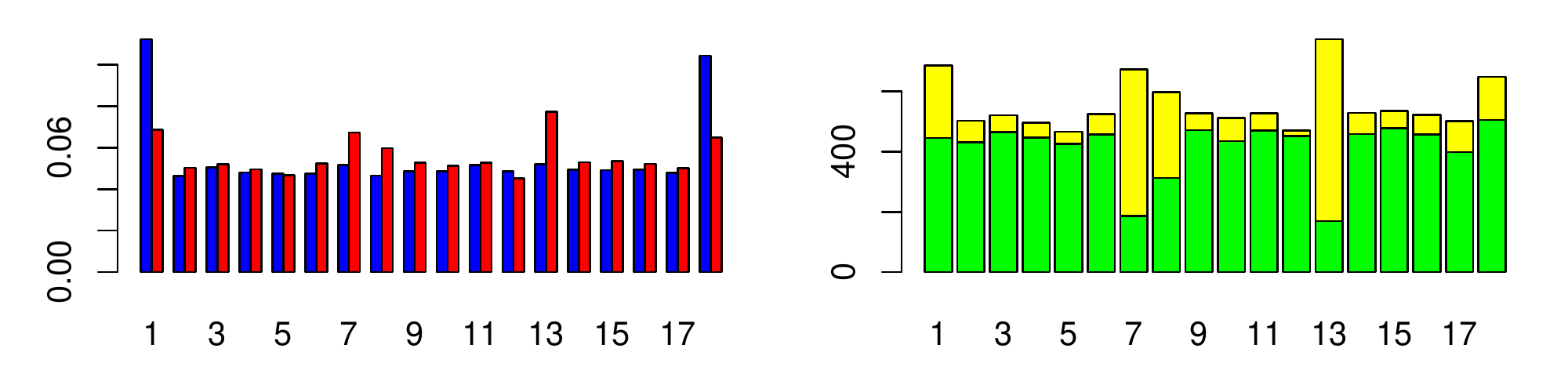} \\
		{\scriptsize 10000 samples} \\
		\includegraphics[scale=0.3]{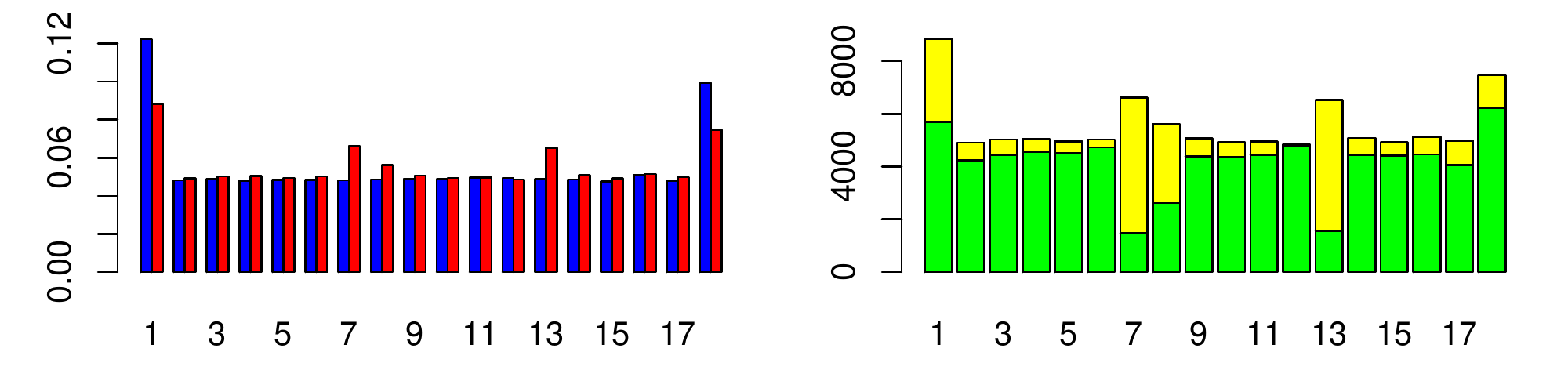} \\
		{\scriptsize 100000 samples} \\
	\end{minipage}
	\caption{HMM, predicting the 0th and 17th state}
	\label{fig:hmm-x}
\end{figure}

As can be seen in the plots, Adaptive LMH (black) exhibits
faster convergence over the whole range of evaluation, requiring
half as many samples as LMH (cyan) to achieve the same approximation, with
the median of LMH above the 75\% quantile of Adaptive LMH.

In addition, the bar plots show unit rewards and sample counts for
different random choices, providing an insight on the adaptive
behavior of AdLMH. On the left-hand bar plots, red bars are
normalized unit rewards, and blue bars are normalized sample
counts. On the right-hand bar plots, the total height of a bar is
the total sample count, with green section corresponding to the
accepted, and yellow to the rejected samples. At $1\,000$ samples,
the unit rewards have not yet converged, and exploration supersedes
exploitation: random choices with lower acceptance rate are selected
more often (choices 7, 8 and 13 corresponding to states 6, 7 and
12). At $10\,000$ samples, the unit rewards become close to their final
values, and choices 1 and 18, immediately affecting the predicted
states, are selected more often.  At $100\,000$ samples, the
unit rewards converge, and the sample counts correspond closely to
the equilibrium state outlined in Lemma~\ref{thm:lemma-rate}.

The second case study is estimation of
hyperparameters of a Gaussian Process. We define a Gaussian Process
of the form
\begin{align*}
    f\sim&\mathcal{GP}(m,k),\\
    \mbox{where }m(x)=&ax^2+bx+c,\quad k(x,x′)=de^{−\frac {(x - x')^2} {2g}}.
\end{align*}
The process has five hyperparameters, $a, b, c, d, g$. The program
infers the posterior values of the hyperparameters by maximizing
marginal likelihood of 6 observations $(0.0,0.5)$, $(1.0,0.4)$,
$(2.0,0.2)$, $(3.0,-0.05)$, $(4.0,-0.2)$, and $(5.0,0.1)$.
Parameters $a, b, c$ of the mean function are predicted.  Maximum
of KS distances between inferred distributions of each of the
predicted parameters and an approximation of the target distributions
is shown in Figure~\ref{fig:gp}. The approximation was obtained by
running LMH with $2\,000\,000$ samples per restart and 50 restarts,
and then taking each 100th sample from the last $10\,000$ samples of
each restart, 5000 samples total. Just as for
the previous case study, bar plots of unit rewards and sample counts are
shown for  $1000$, $10\,000$, and $100\,000$ samples. 

\begin{figure}[t]
	\begin{minipage}[c]{0.45\linewidth}
		\centering
		\includegraphics[scale=0.4]{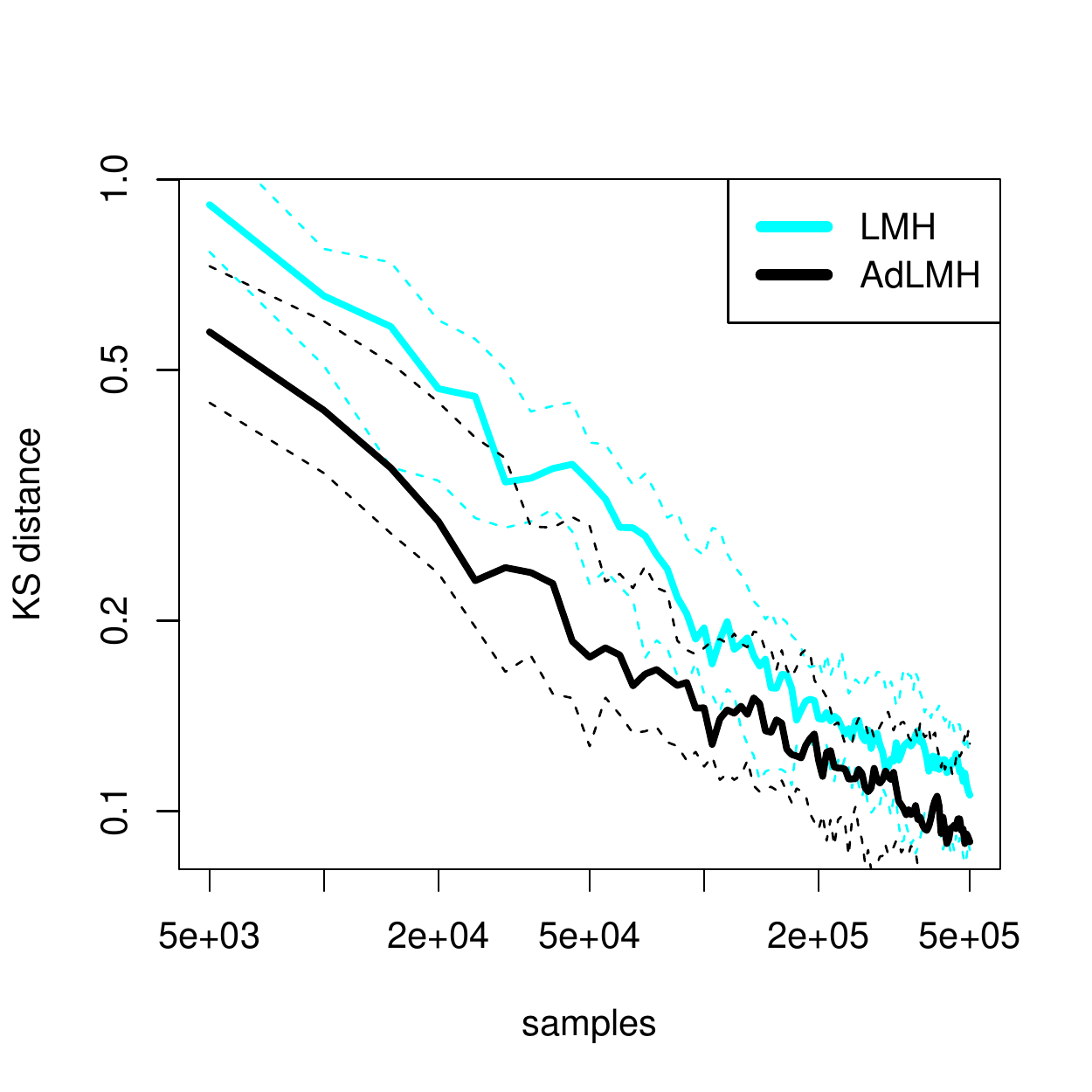} \\
	\end{minipage}
	\begin{minipage}[c]{0.55\linewidth}
		\centering
		\includegraphics[scale=0.3]{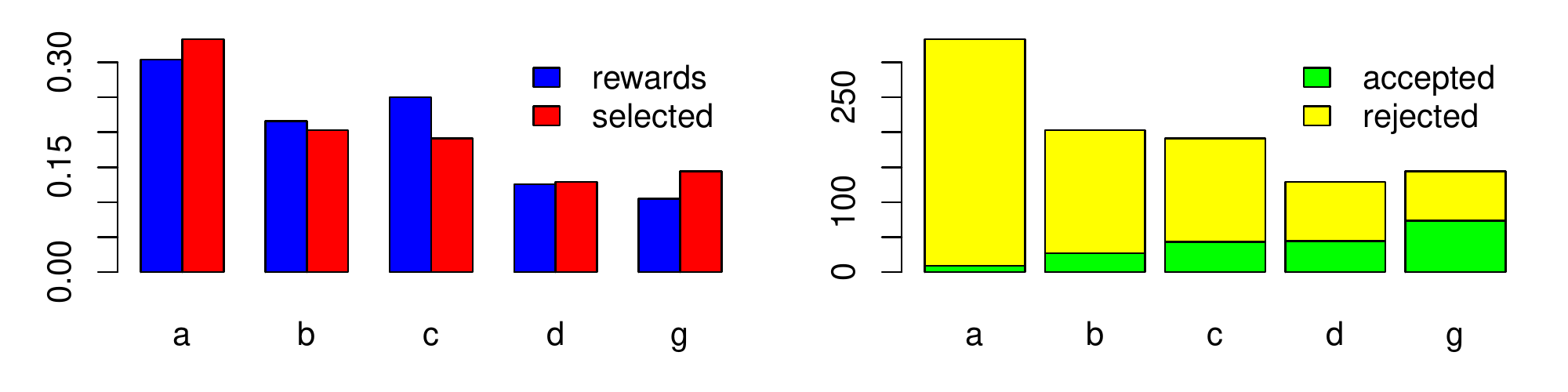} \\
		{\scriptsize 1000 samples} \\
		\includegraphics[scale=0.3]{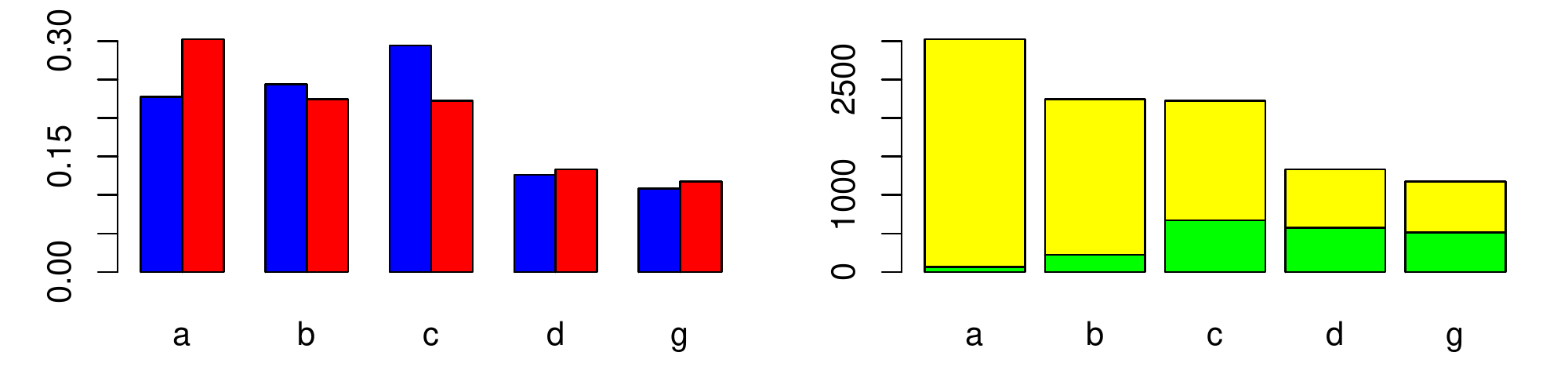} \\
		{\scriptsize 10000 samples} \\
		\includegraphics[scale=0.3]{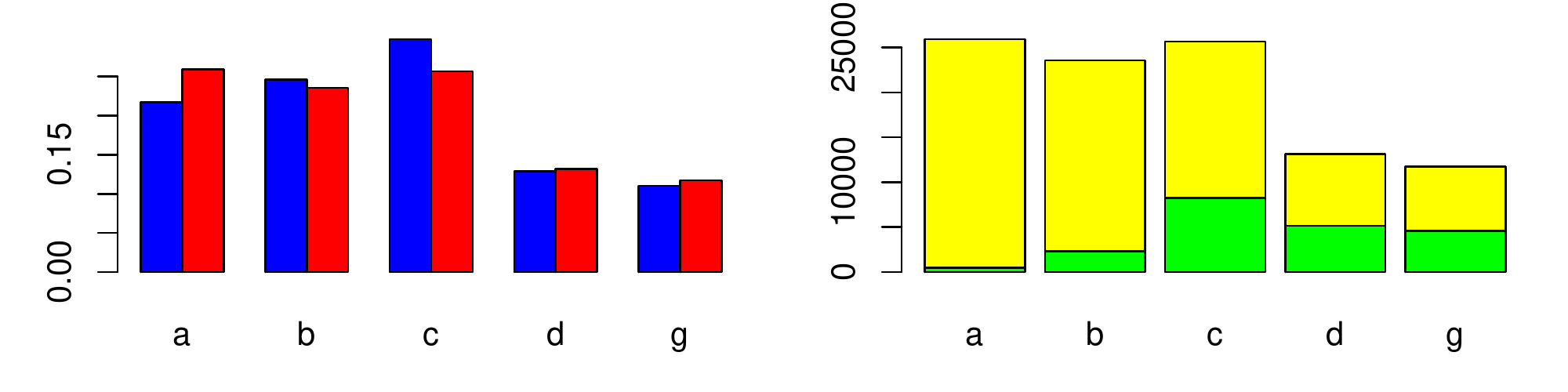} \\
		{\scriptsize 100000 samples}
	\end{minipage}
\caption{Gaussian process hyperparameter estimation}
\label{fig:gp}
\end{figure}

Here as well, Adaptive LMH (black) converges faster over the whole
range of evaluation, outperforming LMH by a factor 2 over the first
$50\,000$ samples. Bar plots of unit rewards and sample counts for
different number of choices, again, show the dynamics of sample
allocation among random choices. Choices $a$, $b$, and $c$ are
predicted, while choices $d$ and $g$ are required for inference but
not predicted.  Choice $a$ has the lowest acceptance rate (ratio
between the total height of the bar and the green part on the
right-hand bar plot), but the unit reward is close the unit reward of
choices $b$ and $c$. At $1\,000$ samples, choice $a$ is selected with
the highest probability. However, close to the converged state, at
$100\,000$ samples, choices $a$, $b$, and $c$ are selected with
similar probabilities. At the same time, choices 4 and 5 are selected
with a lower probability. Both the exploration-exploitation dynamics
for choices $a$--$b$ and probability matching of selection
probabilities among all choices secure improved convergence.

The third case study involves a
larger amount of data observed during each simulation of a
probabilistic program. We use the well-known Iris dataset~\cite{L96}
to fit a model of classifying a given flower as of the species Iris
setosa, as opposite to either Iris virginica or Iris versicolor.
Each record in the dataset corresponds to an observation. For each
observation, we define a feature vector $\vec{x}$  and an indicator
variable $z_i$, which is 1 if and only if the observation is of an
Iris setosa. We fit the model with five regression coefficients $\beta_1,\ldots,\beta_5$,
defined as
\begin{align*}
    \sigma^2 &\sim \mathrm{Inv\-Gamma}(1,1),\\
    \beta_j &\sim \mathrm{Normal}(0, \sigma), \\
    p(z_i=1) &= \frac{1}{1 + e^{-\beta^T{x}}}.
\end{align*}
To assess the convergence, we perform shuffle split leave-2-out cross validation on the
dataset, selecting one instance belonging to the species Iris setosa and one
belonging to a different species for each run of the inference algorithm.
The classification error is shown in Figure~\ref{fig:lr-iris} over 100 runs of LMH
and Adaptive LMH.

\begin{figure}[t]
	\centering
	\includegraphics[scale=0.45]{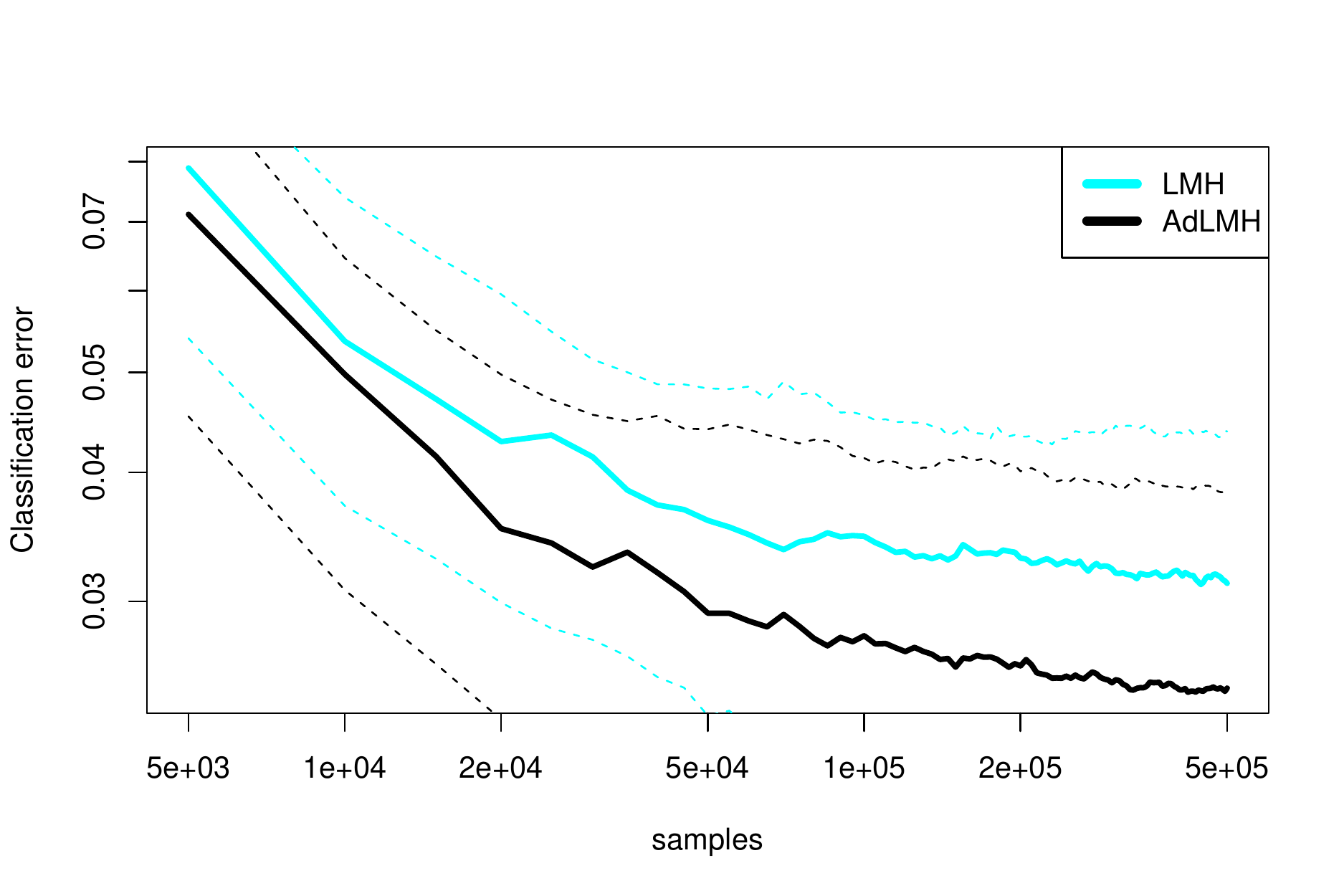} \\
	\caption{Logistic regression on Iris dataset.}
	\label{fig:lr-iris}
\end{figure}
The results are consistent with other case studies: Adaptive LMH
exhibits a faster convergence rate, requiring half as many samples
to achieve the same classification accuracy as LMH.

\begin{figure}[t]
	\centering
    \includegraphics[scale=0.4]{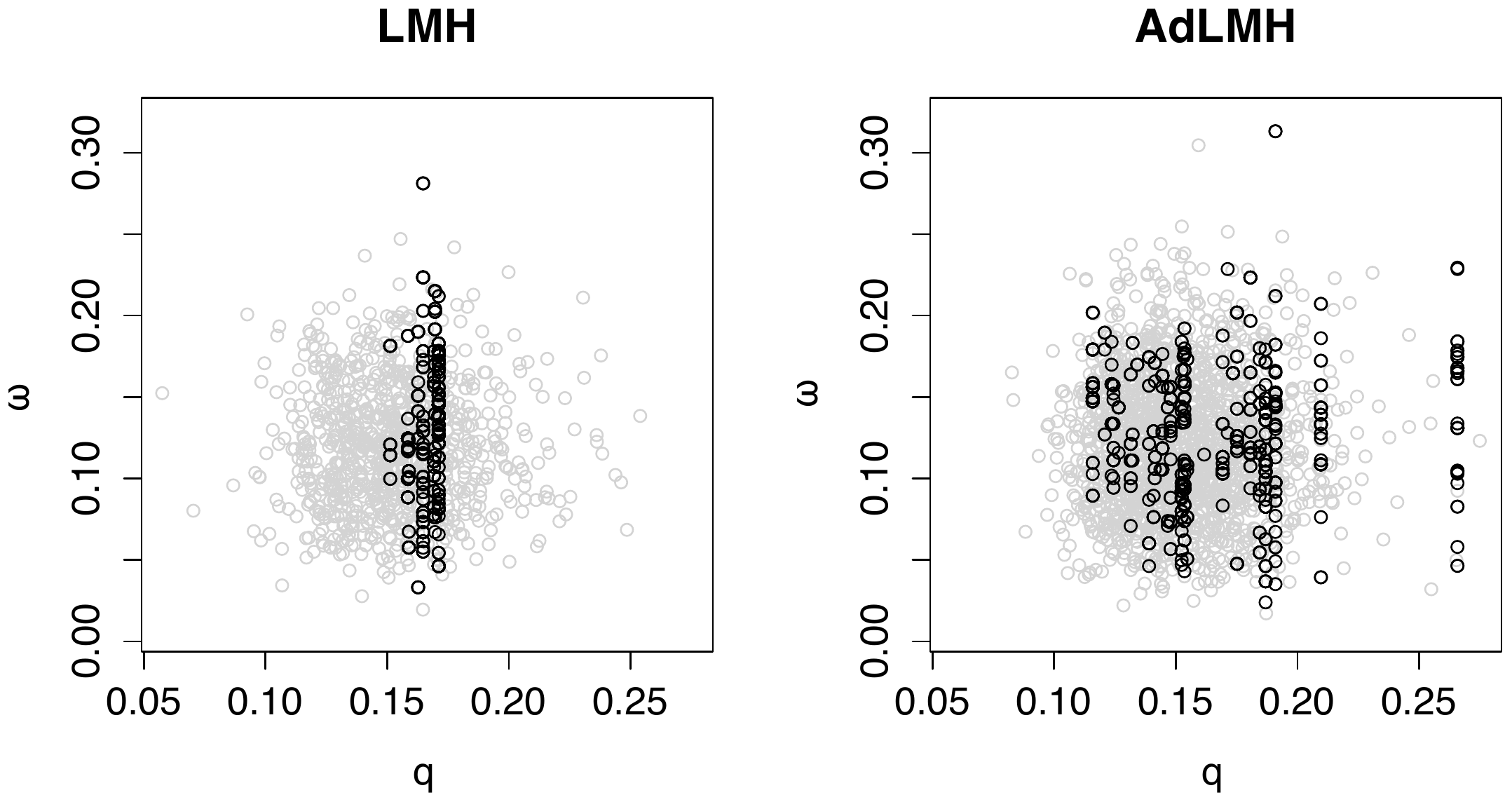} 
	\caption{Kalman filter, 500 samples after 10\,000 samples of
	burn-in.}
\label{fig:kalman-filter}
\end{figure}

As a final case study we consider a linear dynamical system (i.e. a Kalman smoothing problem) that was previously described in \cite{MYM+15}
\begin{align*}
    \vec{x}_t 
    &\sim 
    {\rm Norm}(\vec{A} \cdot \vec{x}_{t-1}, \vec{Q})
    ,
    &
    \y_t 
    &\sim
    {\rm Norm}(\vec{C} \cdot \vec{x}_{t}, \vec{R})
    .
\end{align*}
In this problem we assume that 16-dimensional observations $\vec{y}_t$ are conditioned on 2-dimensional latent states $\vec{z}_t$. 
We impose additional structure by assuming that the transition matrix $\vec{A}$ is a simple rotation with angular velocity $\omega$, whereas the transition covariance $\vec{Q}$ is a diagonal matrix with coefficient $q$,
\begin{align*}
    \vec{A} 
    &= 
    \left[
        \begin{array}{cc}
            ~\cos \omega~
            & 
            ~-\sin \omega~
            \\
            ~\sin \omega~
            &
            ~\cos \omega~
        \end{array}
    \right]
    ,
    &
    \vec{Q}
    &=
    \left[
        \begin{array}{cc}
            ~q~
            & 
            0
            \\
            0
            &
            ~q~
        \end{array}
    \right]
    .
\end{align*}
We predict posterior values for $\omega$, and $q$ in a setting where $\vec{C}$ and $\vec{R}$ are assumed known, under mildly informative priors $\omega \sim {\rm Gamma}(10,2.5)$ and $q \sim {\rm Gamma}(10,100)$.
Posterior inference is performed conditioned on a simulated sequence $\vec{y}_{1:T}$ of $T=100$ observations, with $\omega^* = 4 \pi / T$, and $q^*=0.1$. The observation matrix $\vec{C}$ and covariance $\vec{R}$ are sampled row-wise from symmetric Dirichlet distributions with  parameters $c=0.1$, and $r=0.01$ respectively.

Figure \ref{fig:kalman-filter} shows a qualitative evaluation the mixing rate in the form of 500 consecutive samples $(\omega,q)$ from an LMH and AdLMH chain after 10\,000 samples of burn-in. The LMH sequence exhibits good mixing over $\omega$ but is strongly correlated in $q$, whereas the AdLMH sequence obtains a much better coverage of the space.

To summarize, Adaptive LMH consistently attained faster convergence
than LMH, measured by differences between the ongoing output distribution of
the random program and the target independently obtained
distribution, assessed using various metrics. Variable selection probabilities
computed by Adaptive LMH are dynamically adapted during the inference, combining
exploration of the model represented by the probabilistic program and
exploitation of influence of random variables on program output. 

\section{Contribution and Future Work}

In this paper we introduced a new algorithm, Adaptive LMH,
for approximate inference in probabilistic programs. This
algorithm adjusts sampling parameters based on the output of the
probabilistic program in which the inference is
performed. Contributions of the paper include
\begin{itemize}
  \item A scheme of rewarding random choice based on program
    output.
  \item An approach to propagation of choice rewards to 
    MH proposal scheduling parameters.
  \item An application of this approach to LMH, where the
    probabilities of selecting each variable for modification
    are adjusted.
\end{itemize}
Adaptive LMH was compared to LMH, its non-adaptive counterpart, and was
found to consistently outperform LMH on several probabilistic
programs, while still being almost as easy to implement. The time cost
of additional computation due to adaptation was negligible.

Although presented in the context of a particular sampling algorithm,
the adaptation approach can be extended to other sampling methods. We
believe that various sampling algorithms for probabilistic programming
can benefit from output-sensitive adaptation. Additional potential
for improvement lies in acquisition of dependencies between predicted
expressions and random variables. Exploring alternative approaches
for guiding exploration-exploitation compromise, in particular,
based on Bayesian inference, is another promising research
direction.

Overall, output-sensitive approximate inference appears to bring clear
advantages and should be further explored in the context of
probabilistic programming models and algorithms.

\section{Acknowledgments}
 
This work is supported under DARPA PPAML through the U.S. AFRL under
Cooperative Agreement number FA8750-14-2-0004.

\bibliographystyle{splncs03}
\bibliography{refs}


\end{document}